\documentclass[runningheads]{llncs}

\usepackage[utf8]{inputenc}
\usepackage{amssymb}
\usepackage{makeidx}
\usepackage[english]{babel}
\usepackage{graphicx}
\usepackage{float} 
\usepackage[caption = false]{subfig} 
\usepackage{amsfonts,amsmath}
\usepackage{oldgerm}
\usepackage{relsize}

\usepackage{mathrsfs}
\usepackage[active]{srcltx}
\usepackage{verbatim}
\usepackage{enumitem} 
\usepackage[toc,page]{appendix}
\usepackage{aliascnt,bbm}
\usepackage{array}
\usepackage[textwidth=3cm,textsize=footnotesize]{todonotes}
\usepackage{xargs}
\usepackage{cellspace}

\usepackage[Symbolsmallscale]{upgreek}
\usepackage{xstring}

\usepackage{array,multirow,makecell}

\usepackage{dsfont}
\usepackage{aliascnt}
\usepackage{cleveref}
\makeatletter

\crefname{figure}{figure}{figures}
\Crefname{Figure}{Figure}{Figures}

\Crefname{assumption}{\textbf{H}\hspace{-3pt}}{\textbf{H}\hspace{-3pt}}
\crefname{assumption}{\textbf{H}}{\textbf{H}}

\Crefname{assumptionA}{\textbf{A}\hspace{-3pt}}{\textbf{A}\hspace{-3pt}}
\crefname{assumptionA}{\textbf{A}}{\textbf{A}}

\Crefname{assumptionS}{\textbf{S}\hspace{-3pt}}{\textbf{S}\hspace{-3pt}}
\crefname{assumptionS}{\textbf{S}}{\textbf{S}}

\Crefname{probleme}{\textbf{Problem}\hspace{-3pt}}{\textbf{Problem}\hspace{-3pt}}
\crefname{probleme}{\textbf{Problem}}{\textbf{Problem}}

\Crefname{assumptionG}{\textbf{G}\hspace{-3pt}}{\textbf{G}\hspace{-3pt}}
\crefname{assumptionG}{\textbf{G}}{\textbf{G}}

\makeatother

\usepackage{tikz}
\usepackage{pgfplots}
\usepackage{pgfkeys}
\usetikzlibrary{spy}
\usetikzlibrary[patterns] 
\usetikzlibrary{calc,decorations.pathreplacing} 
\usetikzlibrary{shadows.blur} 
\usetikzlibrary{shapes.symbols}

\newcommand*{\figuretitle}[1]{\textit{\textbf{#1.}}}

\newcommand{\wrt}{with respect to}

\def\x{{ \boldsymbol x}}

\def\y{{\boldsymbol y}}
\def\z{{\boldsymbol z}}

\newcommandx{\norm}[2][1=]{\ifthenelse{\equal{#1}{}}{\left\Vert #2 \right\Vert}{\left\Vert #2 \right\Vert^{#1}}}
\newcommandx{\normLigne}[2][1=]{\ifthenelse{\equal{#1}{}}{\Vert #2 \Vert}{\Vert #2\Vert^{#1}}}

\def\rset{\mathbb{R}}

\def\rmd{\mathrm{d}}

\def\rme{\mathrm{e}}

\newcommand{\R}{\mathbb R}

\newcommandx{\functionspace}[2][1=+]{\mathbb{F}_{#1}(#2)}

\newcommand{\argmin}{\operatorname*{arg\,min}}

\newcommandx{\VarDeux}[3][3=]{\operatorname{Var}^{#3}_{#1}\left\{#2 \right\}}

\newcommand{\LeftEqNo}{\let\veqno\@@leqno}

\newcommand{\N}{\ensuremath{\mathbb{N}}}

\newcommand{\PE}{\mathbb{E}}

\newcommandx{\Vnorm}[2][1=V]{\| #2 \|_{#1}}
\newcommandx{\VnormEq}[2][1=V]{\left\| #2 \right\|_{#1}}

\newcommand{\parenthese}[1]{\left(#1 \right)}

\newcommand{\parentheseDeux}[1]{\left[ #1 \right]}
\newcommand{\defEns}[1]{\left\lbrace #1 \right\rbrace }

\newcommandx\probaMarkovTilde[2][2=]
{\ifthenelse{\equal{#2}{}}{{\widetilde{\mathbb{P}}_{#1}}}{\widetilde{\mathbb{P}}_{#1}\left[ #2\right]}}

\newcommand{\expe}[1]{\PE \left[ #1 \right]}

\newcommand{\bigO}{\ensuremath{\mathcal O}}

\newcommand\numberthis{\addtocounter{equation}{1}\tag{\theequation}}

\def\ie{\textit{i.e.}}
\def\as{\textit{a.s}}

\def\eqsp{\;}

\newcommandx{\weight}[2][2=n]{\omega_{#1,#2}^N}

\newcommand{\ball}[2]{\operatorname{B}(#1,#2)}

\def\as{\ensuremath{\text{a.s.}}}

\newcommandx\sequence[3][2=,3=]
{\ifthenelse{\equal{#3}{}}{\ensuremath{\{ #1_{#2}\}}}{\ensuremath{\{ #1_{#2}, \eqsp #2 \in #3 \}}}}
\newcommandx\sequenceD[3][2=,3=]
{\ifthenelse{\equal{#3}{}}{\ensuremath{\{ #1_{#2}\}}}{\ensuremath{( #1)_{ #2 \in #3} }}}

\newcommandx{\sequencen}[2][2=n\in\N]{\ensuremath{\{ #1_n, \eqsp #2 \}}}
\newcommandx\sequenceDouble[4][3=,4=]
{\ifthenelse{\equal{#3}{}}{\ensuremath{\{ (#1_{#3},#2_{#3}) \}}}{\ensuremath{\{  (#1_{#3},#2_{#3}), \eqsp #3 \in #4 \}}}}
\newcommandx{\sequencenDouble}[3][3=n\in\N]{\ensuremath{\{ (#1_{n},#2_{n}), \eqsp #3 \}}}

\newcommand{\opnorm}[1]{{\left\vert\kern-0.25ex\left\vert\kern-0.25ex\left\vert #1 
    \right\vert\kern-0.25ex\right\vert\kern-0.25ex\right\vert}}

\def\Id{\operatorname{Id}}

\newcommandx{\CPE}[3][1=]{{\mathbb E}_{#1}\left[\left. #2 \middle \vert #3 \right. \right]} 
\newcommandx{\CPVar}[3][1=]{\mathrm{Var}^{#3}_{#1}\left\{ #2 \right\}}
\newcommand{\CPP}[3][]
{\ifthenelse{\equal{#1}{}}{{\mathbb P}\left(\left. #2 \, \right| #3 \right)}{{\mathbb P}_{#1}\left(\left. #2 \, \right | #3 \right)}}

\def\scrB{\mathscr{B}}

\def\scrF{\mathscr{F}}
\def\scrG{\mathscr{G}}
\def\scrP{\mathscr{P}}

\newcommandx{\osc}[2][1=]{\mathrm{osc}_{#1}(#2)}

\def\Id{\operatorname{Id}}

\def\n{\mathrm{n}}

\def\a{a}
\def\b{b}
\def\c{c}

\def\y{y}
\def\z{z}

\newcommand\coupling[2]{\Gamma(\mu,\nu)}

\renewcommand{\geq}{\geqslant}
\renewcommand{\leq}{\leqslant}

\def\vareps{\varepsilon}

\newcommandx{\KL}[2]{\text{KL}\left( #1 | #2 \right)}

\usepackage{graphicx}
\begin{document}
\title{Macrocanonical Models for Texture Synthesis}
%
%
\author{Valentin De Bortoli\inst{1}\and
Agnès Desolneux\inst{1} \and
Bruno Galerne\inst{2} \and
Arthur Leclaire\inst{3}}
\authorrunning{V. De Bortoli et al.}
%
\institute{Centre de mathématiques et de leurs applications, CNRS, ENS 
Paris-Saclay, Université Paris-Saclay, 94235, Cachan cedex, France. \and
  Institut Denis Poisson, Universit\'{e} d'Orléans, Universit\'{e} de Tours, CNRS \and
Univ. Bordeaux, IMB, Bordeaux INP, CNRS, UMR 5251, F-33400 Talence, France}
\maketitle              
\begin{abstract}
  In this article we consider macrocanonical models for texture synthesis.
  In these models samples are generated given an input texture image and a set of
  features which should be matched in expectation. It is known that if the images
  are quantized, macrocanonical models are given by Gibbs measures, using the
  maximum entropy principle. We study conditions under which this result extends to
  real-valued images. If these conditions hold, finding a macrocanonical model amounts
  to minimizing a convex function and sampling from an associated Gibbs measure. We analyze an
  algorithm which alternates between sampling and minimizing. We present experiments with
  neural network features and study the drawbacks and advantages of using this sampling scheme.

\end{abstract}
%
%

\section{Introduction}

In image processing a texture can be defined as an image which contains repetitive patterns but also randomness
in the pattern placement or in the pattern itself. This vague and unformal definition covers a large class
of images such as the ones of terrain, plants, minerals, fur and skin.
Exemplar-based texture synthesis aims at synthesizing new images of arbitrary size which have the same perceptual characteristics
as a given input texture. It is a challenging task to give a mathematical
framework which is not too restrictive, thus describing many texture images, and not too broad, so that
computations are numerically feasible. In the literature two classes of exemplar-based texture synthesis algorithms have been considered:
the parametric and the non-parametric texture algorithms. Non-parametric texture methods do not rely on an explicit
image model in order to produce outputs. For instance copy-paste algorithms such as \cite{efros1999texture} 
fill the output image with sub-images from the input. Another example is given by \cite{galerne2018texture} in which the authors apply optimal transport tools in a multiscale patch space.

In this work we focus on parametric exemplar-based texture synthesis algorithms.
In contrast to the non-parametric approach they provide an explicit image model. 
Output textures are produced by sampling from this image model. In order to derive such a model perceptual features have to be carefully selected along with a corresponding sampling algorithm.
There have been huge progress in both directions during the last twenty years.

First, it should be noted that textures which do not exhibit long-range correlations and are well described
by their first and second-order statistics can be modeled with Gaussian random fields 
\cite{vanwijk1991spotnoise}, \cite{galerne2011random}.
These models can be understood as maximum entropy distributions given a mean and a covariance matrix. Their simplicity allows for fast sampling as well as good mathematical understanding of the model. However, this simplicity also restricts the class of textures which can be described. Indeed, given more structured inputs,
these algorithms do not yield satisfactory visual results. It was already noted by Gagalowicz \cite{gagalowicz1986model} that first and second-order statistics are not sufficient to synthesize real-world textures images.
In \cite{cano1988hierarchical} the authors remark that multiscale features capture perceptual characteristics. Following this idea, algorithms based on steerable pyramids \cite{heeger1995pyramid}, wavelet coefficients \cite{portilla2000parametric} or wavelet coefficients combined with geometrical
properties \cite{peyre2010grouplet} provide good visual results for a large class of textures. Using Convolutional
Neural Networks (CNN), and especially the VGG model \cite{simonyan2014vgg}, Gatys et al. in \cite{gatys2015texture} obtain striking visual results using Gram matrices computed on the layers of the neural network. All these models are called microcanonical textures according to Bruna and Mallat \cite{bruna2018multiscale}, in the sense that they approximately match statistical constraints almost surely (\as). Indeed, the previously  introduced algorithms start from a noisy input containing all the randomness of the process, then use a (deterministic) gradient descent (or any other optimization algorithm)
in order to fit constraints.%

On the other hand, models relying on constraints in expectation have been considered in \cite{zhu1998filters}.
They correspond to macrocanonical textures according to \cite{bruna2018multiscale}.
They have the advantage to be described by exponential distributions and thus, since their distribution can be made
explicit up to some parameters, standard statistical tools can be used for mathematical analysis. However, as noted in \cite{bruna2018multiscale} they often
rely on Monte Carlo algorithms which can be slow to converge. Zhu et al. \cite{zhu1998filters} consider a bank of linear and non-linear filters in order to build an exponential model. Texture images are supposed to be quantized and a Gibbs sampler on each pixel is used
in order to update the image. In \cite{lu2015learning} the authors propose to use first-order statistics computed on CNN outputs. They also suggest to use a Langevin algorithm in order to update the whole image at each iteration. It has also been remarked in \cite{ulyanov2016texture} that specific Generative Adversarial Networks (GAN) \cite{jetchev2016texture} which produce satisfying outputs from a perceptual point of view but lack mathematical understanding can be embedded in an expectation constraint model using the Maximum Mean Discrepancy principle \cite{gretton2006kernel}.

Our contribution is both theoretical and experimental. After recalling the definition of microcanonical models in Section \ref{sec:micr-models} we give precise conditions under which macrocanonical models, \ie \ maximum entropy models, can be written as exponential distributions in Section \ref{sec:macr-models}. In Section \ref{sec:feature_ex}, we examine how 
these conditions translate into a neural network model. Assuming that the maximum entropy principle is satisfied we then turn to the search of the parameters in such a model. 
The algorithm we consider, which was already introduced without theoretical proof of convergence in \cite{lu2015learning}, relies on the combination of a gradient descent dynamic, see Section \ref{sec:maximizing-entropy}, and a discretized Langevin dynamic, see Section \ref{sec:sampling-from-gibbs}. Using new results on these Stochastic Optimization with Unadjusted Kernel (SOUK) algorithms \cite{debortoli2019souk} 
convergence results hold for the algorithm introduced in \cite{lu2015learning}, see Section \ref{sec:combining-dynamics}. We then 
provide experiments and after assessing the empirical convergence of our algorithm in Section \ref{sec:empir-conv-sampl} we investigate our choice of models in Section \ref{sec:neur-netw-feat-2}. We draw the conclusions and limitations of our work in Section \ref{sec:concl-persp}.


\section{Maximum entropy models}
\subsection{Microcanonical models}
\label{sec:micr-models}
Let $x_0$ be a given input texture. For ease of exposition we consider that $x_0 \in \rset^d$, with $d \in \N$, but our results extend to images and color images. We aim at sampling $x$ from a probability distribution satisfying $f(x) \approx f(x_0)$, where $f: \rset^d \to  \rset^p$ are some statistics computed over the images. However if such a probability distribution exists it is not necessarily unique. In order for the problem to be well-posed we introduce a reference function $J: \rset^d \to (0,+\infty)$ such that $\int_{\rset^d} J(x) \rmd \lambda (x) < + \infty$ and we associate to $J$ a probability distribution $\Pi_J$ such that $\Pi_J(A) = Z_J^{-1} \int_{A} J(x) \rmd \lambda(x)$ with $Z_J =\int_{\rset^d} J(y) \rmd \lambda(y)$ for any $A \in \scrB(\rset^d)$, the Borel sets of $\rset^d$. Let $\scrP$ be the set of probability distributions over $\scrB(\rset^d)$. If $ \Pi \in \scrP$ is absolutely continuous \wrt \ the Lebesgue measure $\lambda$ we denote by $\frac{\rmd\Pi}{\rmd \lambda}$ the probability density function of $\Pi$. We introduce the $J$-entropy, see \cite{jaynes1957info}, $H_J: \scrP \to [-\infty, +\infty)$ such that for any $\Pi \in \scrP$
\begin{equation*}
  H_J(\Pi) = \begin{cases} - \int_{\rset^d} \log\parentheseDeux{ \frac{\rmd \Pi}{\rmd \lambda}(x) J(x)^{-1}}  \frac{\rmd \Pi}{\rmd \lambda}(x) \rmd \lambda(x) & \text{if  $\frac{\rmd \Pi}{\rmd \lambda}$ exists \ ;}\\ -\infty & \text{otherwise.} \end{cases}
\end{equation*}
The quantity $H_J$ is closely related to the Kullback-Leibler divergence between $\Pi$ and $\Pi_J$. We recall that, if $\Pi$ is absolutely continuous \wrt \ $\lambda$, we have $\KL{\Pi}{\Pi_J} = \int_{\rset^d} \log\parentheseDeux{\frac{\rmd \Pi}{\rmd \lambda}(x) \frac{\rmd \Pi_J}{\rmd \lambda}(x)^{-1} }  \frac{\rmd \Pi}{\rmd \lambda}(x) \rmd \lambda(x)$, and $+\infty$ otherwise. 
Since $\frac{\rmd \Pi_J}{\rmd \lambda}(x) = Z_J^{-1} J(x)$ we obtain that
for any $\Pi \in \scrP$, $H_J(\Pi) = - \KL{\Pi}{\Pi_J} + \log(Z_J)$.
The following definition gives a texture model for which statistical constraints are met \as
\begin{definition}
  The probability distribution function $\widetilde{\Pi} \in \scrP$ is a \emph{microcanonical model} associated with the exemplar texture $x_0 \in \rset^d$, statistics $f: \rset^d \to  \rset^p$ and reference $J$ if 
  \begin{equation}
    \label{eq:micro_model}
    H_J(\widetilde{\Pi}) = \max \defEns{H_J(\Pi), \ \Pi \in \scrP, \ f(X) = f(x_0) \ \as \ \text{if} \ X \sim \Pi} \eqsp .
  \end{equation}
\end{definition}
Most algorithms which aim at finding a microcanonical model apply a gradient descent algorithm on the function $x\mapsto \| f(x) - f(x_0) \|^2$ starting from an initial white noise. The intuition behind this optimization procedure is that the entropy information is contained in the initialization and the constraints are met asymptotically. There exists few theoretical work on the subject with the remarkable exception of \cite{bruna2018multiscale} in which the authors prove that under technical assumptions the limit distribution has its support on the set of constrained images, \ie \ the constraints are met asymptotically, and provide a lower bound on its entropy.
\subsection{Macrocanonical models}
\label{sec:macr-models}
Instead of considering \as \ constraints as in \eqref{eq:micro_model} we can consider statistical constraints
in expectation. This model was introduced by Jaynes in \cite{jaynes1957info} and  formalized in the context of image processing by Zhu et al. in \cite{zhu1998filters}. 
\begin{definition}
  The probability distribution function $\widetilde{\Pi} \in \scrP$ is a \emph{macrocanonical model} associated with the exemplar texture $x_0 \in \rset^d$, statistics $f: \rset^d \to  \rset^p$ and reference $J$ if 
  \begin{equation}
    \label{eq:macro_model}
    H_J(\widetilde{\Pi}) = \max \defEns{H_J(\Pi), \ \Pi \in \scrP, \ \Pi(f) = f(x_0) } \eqsp ,
  \end{equation}
  where $\Pi(f) = \mathbb{E}_{\Pi}(f)$.
\end{definition}
Macrocanonical models can be seen as a relaxation of the microcanonical ones.
A link between macrocanonical models and microcanonical models is highlighted by Bruna and Mallat in \cite{bruna2018multiscale}. They show that for some statistics, macrocanonical and microcanonical models have the same limit when the size of the image goes to infinity. This transition of paradigm has important consequences from a statistical point of view. First, the constraints in \eqref{eq:macro_model} require only the knowledge of the expectation of $f$ under a probability distribution $\Pi$. Secondly, in Theorem \ref{thm:maxent}  we will show that the macrocanonical model can be written as a Gibbs measure, \ie \ $\frac{\rmd \widetilde{\Pi}}{\rmd \lambda}(x) \propto \exp(-\langle \tilde{\theta}, f(x) -f(x_0) \rangle)J(x)$ for some $\tilde{\theta} \in \rset^p$. Given $\theta \in \rset^p$, when it is defined we denote by $\Pi_{\theta}$ the probability distribution defined by 
\begin{equation*}
  Z(\theta) = \int_{\rset^d} \rme^{-\langle \theta, f (x) - f(x_0) \rangle}J(x) \rmd \lambda(x) \quad \text{and} \quad \frac{\rmd \Pi_{\theta}} {\rmd \lambda}(x) =  \frac{\rme^{-\langle \theta, f(x) -f(x_0) \rangle}}{Z(\theta)} J(x) \eqsp .
\end{equation*}
\begin{theorem}[Maximum entropy principle]
  \label{thm:maxent}
  Assume that for any $\theta \in \rset^p$ we have
  \begin{equation}
    \label{eq:condition_f}
    \int_{\rset^d} \rme^{\|\theta \| \| f(x)\|} J(x) \rmd \lambda(x) < +\infty \quad \text{and} \quad \lambda\left( \defEns{ x \in \rset^d,  \ \langle \theta, f(x) \rangle < \langle \theta, f(x_0) \rangle } \right) > 0 \eqsp .
  \end{equation}
  Then there exists $\tilde{\theta} \in \rset^p$ such that $\Pi_{\tilde{\theta}}$ is a macrocanonical model associated with the exemplar texture $x_0 \in \rset^d$, statistics $f$ and reference $J$. In addition, we have
  \begin{equation}
    \label{eq:mini}
    \tilde{\theta} \in \argmin \defEns{\log \parentheseDeux{\int_{\rset^d} \exp(-\langle \theta, f (x) - f(x_0)\rangle)J(x) \rmd \lambda(x)}, \theta \in \rset^p} \eqsp .
  \end{equation}
\end{theorem}
\begin{proof}
  Without loss of generality we assume that $f(x_0) = 0$.
  First we show that there exists $\tilde{\theta} \in \rset^p$ such that $\Pi_{\tilde{\theta}}$ is well-defined and $\Pi_{\tilde{\theta}}(f) = f(x_0)$.
  The first condition in \eqref{eq:condition_f} implies that $Z(\theta) = \int_{\rset^d} \exp(-\langle \theta, f (x) \rangle)J(x) \rmd \lambda(x)$ is defined for all $\theta \in \rset^p$. Let $\theta_0 \in \rset^p$ we have for any $\theta \in \ball{\theta_0}{1}$, the unit ball centered on $\theta_0$, and $i \in \lbrace 1, \dots, p \rbrace$, using that for any $t \in \R$, $t\leq \rme^t$ and the Cauchy-Schwarz inequality,
  \begin{equation*}
    \begin{aligned}
    &\int_{\rset^d} \left|\frac{\partial}{\partial \theta_i}\parentheseDeux{\exp(-\langle \theta, f(x) \rangle} J(x) \right| \rmd \lambda(x) \leq \int_{\rset^d} \| f(x) \| \exp(-\langle \theta, f(x) \rangle) J(x) \rmd \lambda(x) \\
    &\phantom{aaaaaaaaaaaaaaaaaaaaaaaa} \leq \int_{\rset^d} \exp((\| \theta_0 \| +2) \|f(x)\|) J(x) \rmd \lambda(x) < + \infty \eqsp .
    \end{aligned}
  \end{equation*}
Therefore $\theta \mapsto \log(Z)(\theta)$ is differentiable and we obtain that for any $\theta \in \rset^p, \nabla \log(Z)(\theta) = -\mathbb{E}_{\Pi_{\theta}}(f) = - \Pi_{\theta}(f)$. In a similar fashion we obtain that $\log(Z) \in \mathcal{C}^2(\rset^p, \rset)$ and we have
  $\frac{\partial^2 \log(Z)}{\partial \theta_i \partial \theta_j}(\theta) = \Pi_{\theta}(f_i f_j) - \Pi_{\theta}(f_i)\Pi_{\theta}(f_j)$.
The Hessian of $\log(Z)$ evaluated at $\theta$ is the covariance matrix of $f(X)$ where $X \sim \Pi_{\theta}$ and thus is non-negative which implies that $\log(Z)$ is convex.
We also have for any $\theta \in \R^p$ and $t >0$
\begin{align*}
   \log(Z)(t \theta) &=  \log\parentheseDeux{\int_{\rset^d} \exp(-t\langle \theta, f(x) \rangle) J(x) \rmd \lambda(x)} \\
                                         &\geq  \log\parentheseDeux{\int_{\langle \theta, f(x) \rangle < 0} \exp(-t\langle \theta, f(x) \rangle) J(x) \rmd \lambda(x)}
                                           \underset{t \to + \infty}{\longrightarrow} + \infty \eqsp , \numberthis \label{eq:plus_infty}
\end{align*}
where we use the first condition in \eqref{eq:condition_f} and the monotone convergence theorem. Therefore $\log(Z)$ is coercive along each direction of $\rset^p$. Let us show that $\log(Z)$ is coercive, \ie \ for any $M >0$, there exists $R>0$ such that for all $\|\theta \| \geq R$, $\log(Z)(\theta) \geq M$. Suppose that $\log(Z)$ is not coercive then there exists a sequence $(\theta_n)_{n \in \N}$ such that $\lim_{n \to +\infty} \| \theta_n \| = +\infty$ and $\log(Z)(\theta_n)_{n \in \N}$ is upper-bounded by some constant $M \geq 0$. We can suppose that $\theta_n \neq 0$. Upon extracting a subsequence we assume that $(\theta_n / \| \theta_n \|)_{n \in \N}$ admits some limit $\theta^{\star} \in \rset^p$ with $\| \theta^{\star} \| = 1$. Let $M_0 = \max \parentheseDeux{M, \log(Z)(0)}$, we have the following inequality for all $t >0$
\begin{equation*}
  \log(Z)(t\theta^{\star}) \leq \inf_{n \in \N} \parentheseDeux{|\log(Z)(t \theta^{\star}) - \log(Z)(t \theta_n/\|\theta_n\|)| + \log(Z)(t \theta_n/\|\theta_n \|)} \leq M_0 \eqsp ,
\end{equation*}
where we used the continuity of $\log(Z)$ and the fact that for $n$ large enough $t < \| \theta_n \|$ and therefore by convexity $\log(Z)(t \theta_n/\|\theta_n\|) \leq t/\|\theta_n\| \log(Z)(0) + (1 - t/\|\theta_n\|)\log(Z)(\theta_n) \leq M_0$. Hence for all $t >0$, $\log(Z)(t\theta^{\star})$ is bounded which is in contradiction with \eqref{eq:plus_infty}. We obtain that $\log(Z)$ is continuous, convex, coercive and defined over $\rset^p$. This ensures us that there exists $\tilde{\theta}$ such that $\log(Z)(\tilde{\theta})$ is minimal and therefore $\nabla_{\theta} \log(Z) (\tilde{\theta}) = - \Pi_{\tilde{\theta}}(f) = 0$. Note that we have
\begin{equation*}
  H_J(\Pi_{\tilde{\theta}}) = \int_{\rset^d} \langle \tilde{\theta}, f(x) \rangle \frac{\rmd \Pi_{\tilde{\theta}}}{\rmd \lambda}(x)  \rmd \lambda(x) + \log(Z)(\tilde{\theta}) = \log(Z)((\tilde{\theta}) \eqsp .
\end{equation*}
Now let $\Pi \in \scrP$ such that $\Pi(f) = 0$.  If $\Pi$ is not absolutely continuous \wrt \ the Lebesgue measure, then $H_J(\Pi) = -\infty < H_J(\Pi_{\tilde{\theta}})$. Otherwise if $\Pi$ is absolutely continuous \wrt \ the Lebesgue measure we have the following inequality
\begin{equation*}
  \begin{aligned}
  &H_J(\Pi) = - \int_{\rset^d} \log\parentheseDeux{\frac{\rmd \Pi}{\rmd \lambda}(x) J(x)^{-1}} \frac{\rmd \Pi}{\rmd \lambda}(x) \rmd \lambda(x) \\
           &= - \int_{\rset^d} \log\parentheseDeux{\frac{\rmd \Pi}{\rmd \lambda}(x) \left( \frac{\rmd \Pi_{\tilde{\theta}} }{\rmd \lambda}(x) \right)^{-1} } \frac{\rmd \Pi}{\rmd \lambda}(x)   - \log\parentheseDeux{ \frac{\rmd \Pi_{\tilde{\theta}} }{\rmd \lambda}(x)  J(x)^{-1} } \frac{\rmd \Pi}{\rmd \lambda}(x) \rmd \lambda(x)  \\
           &= - \KL{\Pi}{\Pi_{\tilde{\theta}}} + \log(Z)(\tilde{\theta}) \leq \log(Z)(\tilde{\theta}) = H_J(\Pi_{\tilde{\theta}}) \eqsp ,
           \end{aligned}
\end{equation*}
which concludes the proof.
\end{proof}
Theorem \ref{thm:maxent} gives a method for finding the optimal parameters $\theta \in \rset^p$ by solving the convex problem \eqref{eq:mini}. We address this issue in Section \ref{sec:sampling}.
\subsection{Some feature examples}
\label{sec:feature_ex}
In the framework of exemplar-based texture synthesis, $f$ is defined as the spatial statistics
of some image feature. For instance, let $\scrF:  \rset^d \to \prod_{i=1}^p \rset^{d_i}$ 
be a measurable mapping lifting the image $x \in \rset^d$ in a higher-dimensional space $\prod_{i=1}^p \rset^{d_i}$.
Classical examples include wavelet transforms, power transforms or neural network features. Let $(\scrF_i)_{i = 1, \dots, p}$ such that for any $x \in \rset^d$, $\scrF(x) = (\scrF_1(x), \dots, \scrF_p(x))$ and $\scrF_i: \rset^d \to \rset^{d_i}$. Then the statistics $f$ can be defined for any $x \in \rset^d$ as follows
\begin{equation}
  \label{eq:order_1_feat}
  f(x) = \left(d_1^{-1}\sum_{k=1}^{d_1} \scrF_1(x)(k), \dots, d_p^{-1} \sum_{k=1}^{d_p} \scrF_p(x)(k) \right) \eqsp .
\end{equation}
Note that this formulation includes histograms of bank of filters \cite{portilla2000parametric}, wavelet coefficients \cite{peyre2010grouplet} and scattering coefficients \cite{bruna2018multiscale}. The model defined by such statistics is stationary, \ie \ translation invariant, since we perform a spatial summation. In the following we focus on first-order features, which will be used in Section \ref{sec:empir-conv-sampl} to assess the convergence of our sampling algorithm, and neural network features, extending the work of \cite{lu2015learning}. 
\paragraph{Neural Network features.}
\label{sec:neur-netw-feat-1}
We denote by $\mathcal{A}_{n_2,n_1}(\R)$ the vector space of the affine operators from $\rset^{n_1}$ to $\rset^{n_2}$. 
Let $(A_j)_{j \in \lbrace 1, \dots, M\rbrace} \in \prod_{j=1}^M \mathcal{A}_{n_{j+1},n_{j}}(\R)$, where we let $(n_j)_{j \in \lbrace 1,\dots, M+1 \rbrace} \in \N^{M+1}$, with $M \in \N$ and $n_1=d$. Let $\varphi: \rset \to \rset$. We define for any $j \in \lbrace 1, \dots,M \rbrace$, the $j$-th layer feature $\scrG_j: \rset^d \to \rset^{n_j}$ for any $x \in \rset^d$ by 
\begin{equation*}
  \scrG_j (x) = \left( \varphi \circ A_j \circ \varphi \circ A_{j-1} \circ \dots \circ \varphi \circ A_1 \right) (x) \eqsp ,
\end{equation*}
where $\varphi$ is applied on each component of the vectors. Let $p \in \lbrace 1, \dots, M \rbrace$ and $(j_i)_{i \in \lbrace 1, \dots, p \rbrace} \in \lbrace 1, \dots, M \rbrace^p$ then we can define $\scrF$ as in \eqref{eq:order_1_feat} by
\begin{equation*}
  f(x) = \left( n_{j_1}^{-1}\sum_{k=1}^{n_{j_1}} \scrG_{j_1}(x)(k), \dots, n_{j_p}^{-1}\sum_{k=1}^{n_{j_p}} \scrG_{j_p}(x)(k) \right) \eqsp .\end{equation*}
Assuming that $\varphi$, the non-linear unit, is $\mathcal{C}^1(\rset)$ we obtain that $f$ is $\mathcal{C}^1(\rset^d, \rset^p)$. The next proposition gives conditions under which Theorem \ref{thm:maxent} is satisfied. We denote by $df$ the Jacobian of $f$.
\begin{proposition}[Differentiable neural network maximum entropy]
  Let $x_0 \in \rset^d$ and assume that $df(x_0)$ has rank $\min(d,p) = p$.
  In addition, assume that there exists $C \geq 0$ such that for any $x \in \rset$, $|\varphi(x)| \leq C(1+|x|)$. Then the conclusions of Theorem \ref{thm:maxent} hold for any $J(x)=\exp(-\vareps \| x \|^2)$ with $\vareps >0$.
\end{proposition}
\begin{proof}
  The integrability condition is trivially checked since $f$ is sub-linear using that $|\varphi(x)| \leq C(1+|x|)$. Turning to the proof of the second condition, since $f \in \mathcal{C}^1(\rset^d, \rset^p)$ and $df(x_0)$ is surjective we can assert the existence of an open set $\mathsf{U}$ as well as $\Phi \in \mathcal{C}^1(\mathsf{U}, \rset^d)$ with $f(x_0) \in \mathsf{U}$ such that for any $y \in \mathsf{U}$, $f(\Phi(y)) = y$. Now consider $\theta \in \rset^p$. If $\theta \in f(x_0)^{\perp}$ then for $\vareps >0$ small enough $f(x_0) - \vareps \theta \in \mathsf{U}$ and we obtain that $\langle \theta,  f(\Phi(f(x_0) - \vareps \theta)) \rangle = - \vareps \| \theta \|^2 <0$. If $\theta \notin f(x_0)^{\perp}$ then there exists $\vareps >0$ small enough such that $[f(x_0)(1- \vareps), f(x_0)(1+\vareps)] \subset \mathsf{U}$. Then for any $\alpha \in (-\vareps, \vareps)$ we get that $\langle \theta, f(\Phi((1+\alpha)f(x_0))) \rangle - \langle \theta, f(x_0) \rangle = \alpha \langle \theta, f(x_0) \rangle$. By choosing $\alpha >0$, respectively $\alpha <0$, if $\langle \theta, f(x_0) \rangle >0$, respectively  $\langle \theta, f(x_0) \rangle <0$, we obtain that for any $\theta \in \rset^p$, there exists $x \in \rset^d$ such that $\langle \theta, f(x) \rangle < \langle \theta, f(x_0) \rangle$. We conclude using the continuity of $f$.
\end{proof}


\section{Minimization and sampling algorithm}
\label{sec:sampling}
\subsection{Maximizing the entropy}
\label{sec:maximizing-entropy}
In order to find $\tilde{\theta}$ such that $\Pi_{\tilde{\theta}}$ is the macrocanonical model associated with the exemplar texture $x_0$, statistics $f$ and reference $J$ we perform a gradient descent on $ \log(Z)$.
Let $\theta_0 \in \rset^p$ be some initial parameters. We define the sequence $(\theta_n)_{n \in \N}$ for any $n \in \N$ by 
\begin{equation}
  \label{eq:gradient_descent_true_projected}
  \theta_{n+1} = P_{\Theta} \parentheseDeux{\theta_n - \delta_{n+1} \nabla \log(Z)(\theta_n)} = P_{\Theta} \parentheseDeux{\theta_n + \delta_{n+1} \parenthese{\Pi_{\theta_n} (f) - f(x_0)}} \eqsp ,
\end{equation}
where $(\delta_n)_{n \in \N}$ is a sequence of step sizes with $\delta_n \geq 0$ for any $n \in \N$ and $P_{\Theta}$ is the projection over $\Theta$. The introduction of the projection operator $P_{\Theta}$ is a technical condition in order to guarantee the convergence of the algorithm in \Cref{sec:combining-dynamics}. Implementing the algorithm associated to 
\eqref{eq:gradient_descent_true_projected} requires the knowledge of the
moments of the statistics $f$ for any Gibbs measure $\Pi_{\theta}$ with $\theta \in \rset^p$. The more complex the
texture model is the more difficult it is to compute the expectation of the statistics. This expectation can be written as an integral and techniques such as the ones presented in \cite{ogden2015sequential} 
 could be used. 
We choose to approximate this expectation using a Monte Carlo strategy. Assuming that $(X_n)_{n \in \N}$ are samples from $\Pi_{\theta}$, we have that $ n^{-1} \sum_{k=1}^n f(X_k)$ is an unbiased estimator of $\Pi_{\theta}(f)$.
\subsection{Sampling from Gibbs measures}
\label{sec:sampling-from-gibbs}
We now turn to the problem of sampling from $\Pi_{\theta}$. Unfortunately, most of the time there is no easy way to
produce samples from $\Pi_{\theta}$. Nonetheless, using the ergodicity properties of specific Markov Chains we can still come up with estimators of $\Pi_{\theta}(f)$. Indeed, if $(X_n)_{n \in \N}$ is a homogeneous Markov chain with kernel $K$ and invariant probability measure $\Pi_{\theta}$, \ie \ $\Pi_{\theta} K = \Pi_{\theta}$ we obtain under suitable conditions over $f$ and $K$ that $\lim_{n \to +\infty} \expe{n^{-1} \sum_{k=1}^n f(X_k)} = \Pi_{\theta}(f)$. This leads us to consider the following Langevin dynamic for all $n \in \N$
\begin{equation}
  \label{eq:langevin_dynamic}
  X_{n+1} = X_n - \gamma_{n+1} \sum_{i=1}^p \theta_i \nabla f_i (X_n) + \sqrt{2\gamma_{n+1}} Z_{n+1} \qquad \text{and} \quad X_0 \in \rset^d \eqsp ,
\end{equation}
where $(Z_n)_{n \in \N^*}$ is a collection of independent $d$-dimensional zero mean Gaussian random variables with covariance matrix identity and $(\gamma_n)_{n\in \N}$ is a sequence of step sizes with $\gamma_n \geq 0$. The, possibly inhomogeneous, Markov Chain $(X_n)_{n \in \N}$ is associated with the sequence of kernels $(R_{\gamma_n})_{n \in \N}$ with $R_{\gamma_n}(x, \cdot) = \mathcal{N}(x- \gamma_n \sum_{i=1}^p \theta_i \nabla f_i(x), 2 \gamma_n \Id)$. Note that \eqref{eq:langevin_dynamic} is the Euler-Maruyama discretization of the continuous dynamic
$\rmd X_t = - \sum_{i=1}^p \theta_i \nabla f_i (X_t) \rmd t + \sqrt{2} \rmd B_t $
where $(B_t)_{t \geq 0}$ is a $d$-dimensional Brownian motion. 
  \subsection{Combining dynamics}
  \label{sec:combining-dynamics}
  We now combine the gradient dynamic and the Langevin dynamic. This algorithm is referred as Stochastic Optimization with Unadjusted Langevin (SOUL) algorithm in \cite{debortoli2019souk} and is defined for all $n \in \N$ and $ k \in \lbrace 0, m_n -1 \rbrace $ by the following recursion
  \begin{equation*}
    \begin{aligned}
  X_{k+1}^n &= X_k^n - \gamma_{n+1} \sum_{i=1}^p \theta_i \nabla f_i (X_k^n) + \sqrt{2\gamma_{n+1}} Z_{k+1}^n  \eqsp, \text{$X_0^n = X_{m_{n-1}}^{n-1}$, $n \geq 1$ }\eqsp, \\ 
  \theta_{n+1} &= P_{\Theta}\parentheseDeux{\theta_n + \delta_{n+1} m_{n}^{-1} \sum_{k=1}^{m_{n}} (f(X_k^n) - f(x_0))} \eqsp ,
  \end{aligned}
\end{equation*}
with $X_0^0 \in \rset^d$, $\theta_0 \in \rset^p$, $(\delta_n)_{n \in \N}$, $(\gamma_n)_{n \in \N}$ real positive sequences of step sizes and $(m_n)_{n \in \N} \in \N^{\N}$, the number of Langevin iterations. 
In \cite{debortoli2019souk} the authors study the convergence of these combined dynamics.
\vspace{-0.2cm}

\section{Experiments}
\label{sec:neur-netw-feat}
In this section, we present experiments conducted with neural network features. Texture synthesis with these features has been first done in \cite{gatys2015texture} in which the authors compute
Gram matrices, \ie \ second-order information, on different layers of a network. The underlying model
is microcanonical. In \cite{lu2015learning} the authors consider a macrocanonical model with convolutional
neural network features corresponding to the mean of filters at a given layer. In our model we consider the
features described in Section \ref{sec:neur-netw-feat-1} where the linear units and rectifier units are given
by the VGG-19 model \cite{simonyan2014vgg} which contains 16 convolutional layers
. We consider the following settings and notations:
\begin{itemize}
\item Trained \textbf{(T)} or Gaussian \textbf{(G)}: if option T is selected the weights used in the VGG convolutional
  units are given by a classical pretraining for the classification task on the ImageNet dataset \cite{deng2009imagenet}. If option G is selected we replace the pretrained weights with Gaussian random variables such that the weights of each channel and each layer have same mean and same standard deviation.
\item Shallow \textbf{(3)}, Mid \textbf{(6)}, Deep \textbf{(8)}: in our experiments we consider different settings regarding the number of layers and, more importantly, the influence of their depth. In the Shallow (3) setting we consider the linear layers number 3, 4 and 5. In the Mid (6) setting we consider the linear layers number 3, 4, 5, 6, 7 and 11. In the Deep (8) setting we consider the linear layers number 3, 4, 5, 6, 7, 11, 12 and 14.
\end{itemize}
\subsection{Empirical convergence of the sampling algorithm}
\label{sec:empir-conv-sampl}
We assert the experimental convergence of the SOUL algorithm in Figure \ref{fig:im_cv}.
We choose $\delta_n = \bigO(n^{-1})$, $\gamma_n = \bigO(n^{-1})$, $m_n = 1$, $\Theta = \rset^p$ and $\vareps=0$, \ie \ $J = 1$.
The algorithm is robust for these fixed parameters for a large number of images.
Note that even if this case is not covered by the theoretical results of \cite{debortoli2019souk}, the convergence is
improved using these rates. The drawbacks of not using parameter projection ($\Theta = \rset^p$) or image regularization ($\vareps =0$) is that
the algorithm may diverge for some images, see Figure \ref{fig:dv_layers}.

Interestingly, while neural network features capture perceptual details of the texture input they fail
to restore low frequency components such as the original color histogram. In order to alleviate this
problem, we perform a histogram matching of the output image, as in \cite{gatys2017controlling}. In the next section we investigate the advantages of using CNN channel outputs as features.
\begin{figure}[!h]
  \centering
  \hfill
  \subfloat[]{\includegraphics[width=.15\linewidth]{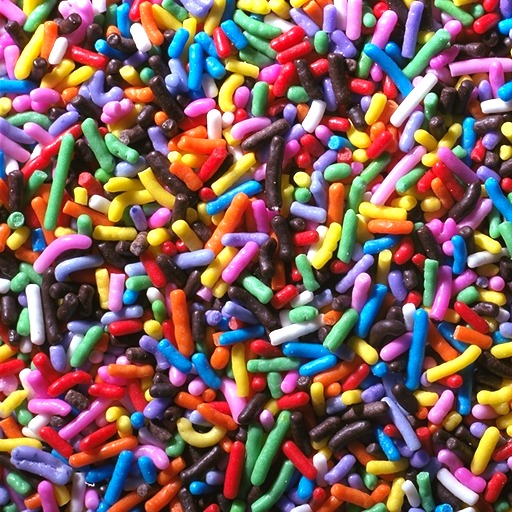}} \hfill  
  \subfloat[]{\begin{tikzpicture}[spy using outlines={rectangle, yellow,magnification=3, connect spies}]
  \node {\pgfimage[interpolate=true,width=.29\linewidth]{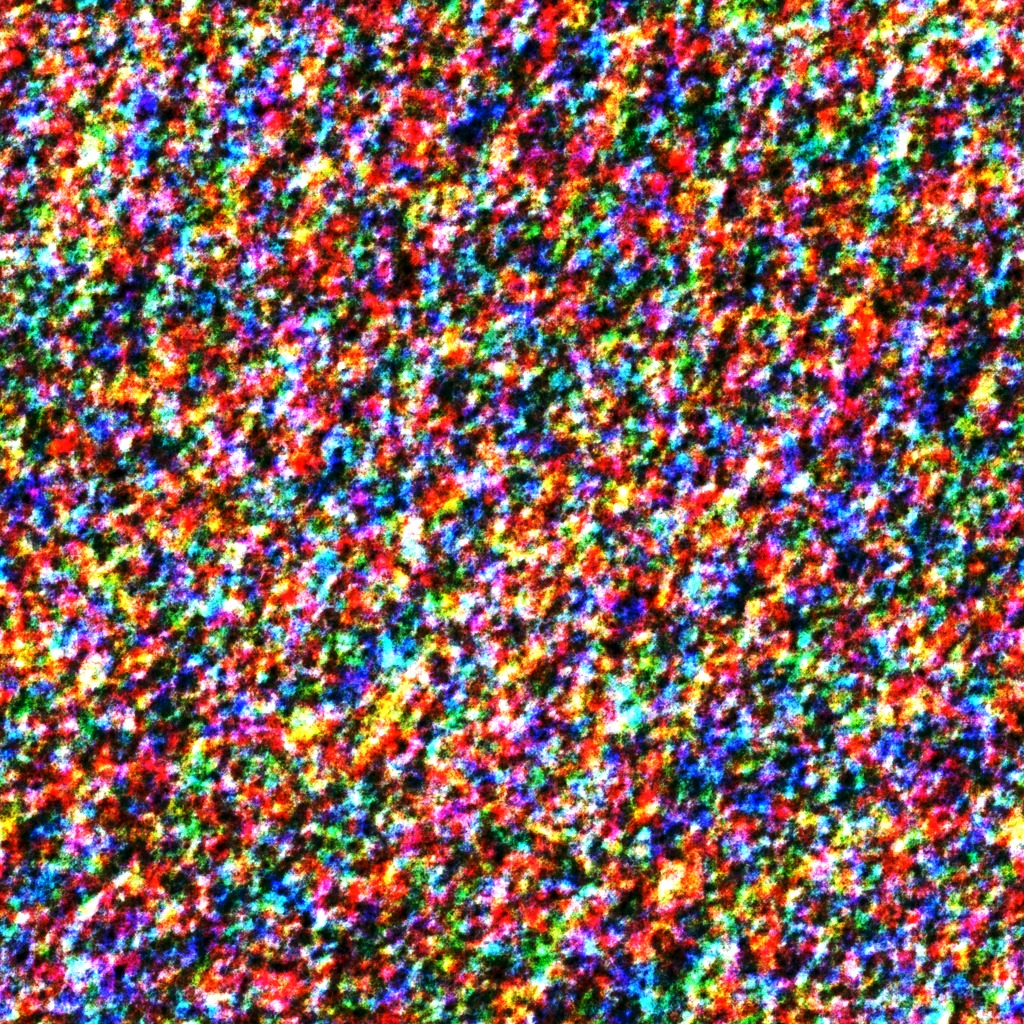}};
 \coordinate (spypoint) at (-0.8, 0.75);
 \coordinate (spyviewer) at (0.8,-0.8);
 \spy[width=2cm,height=2cm] on (spypoint) in node [fill=white] at (spyviewer);
  
\end{tikzpicture}} \hfill
  \subfloat[]{\begin{tikzpicture}[spy using outlines={rectangle, yellow,magnification=3, connect spies}]
  \node {\pgfimage[interpolate=true,width=.29\linewidth]{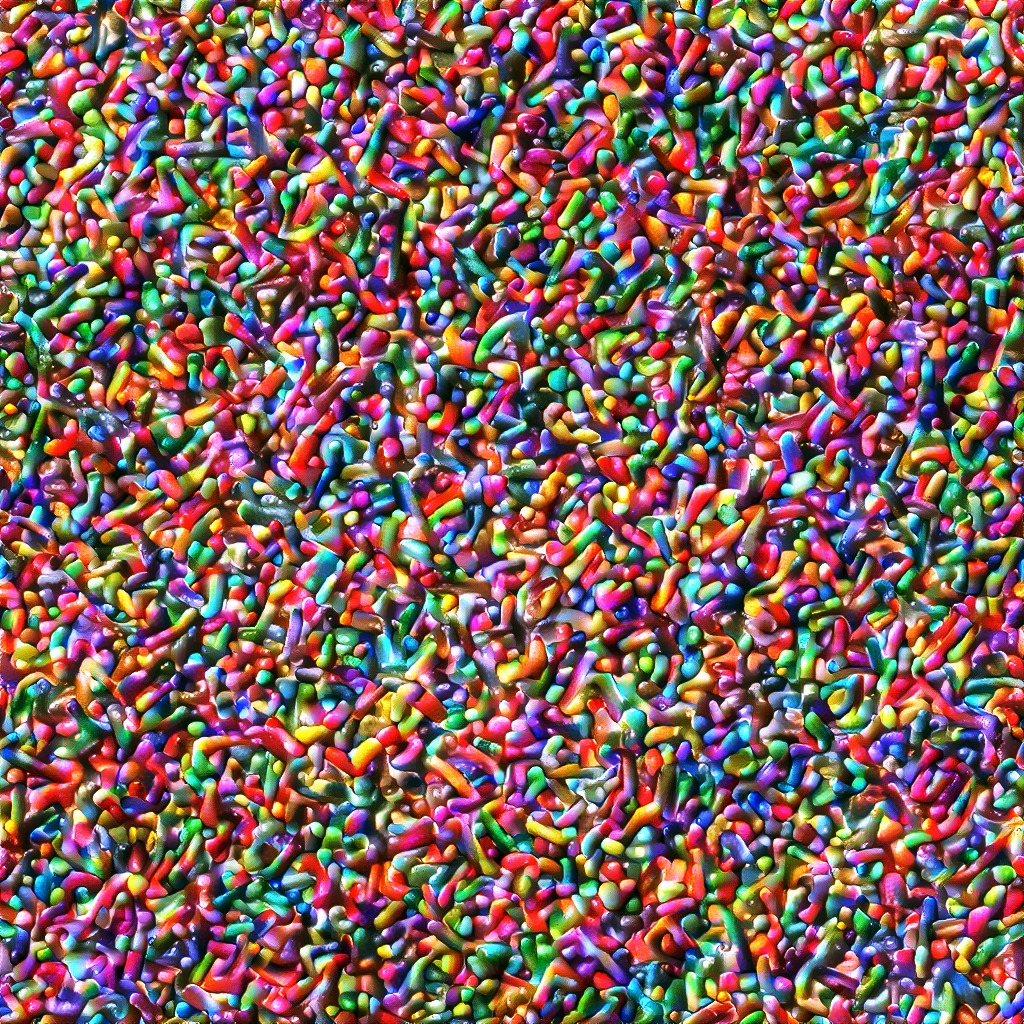}};
 \coordinate (spypoint) at (-0.8, 0.75);
 \coordinate (spyviewer) at (0.8,-0.8);
 \spy[width=2cm,height=2cm] on (spypoint) in node [fill=white] at (spyviewer);
  
\end{tikzpicture}

} \hfill
  \caption{\figuretitle{Empirical convergence} In (a) we present a $512 \times 512$ objective texture. In (b) we show the initialization of our algorithm, a $1024 \times 1024$ Gaussian random fields with same mean and covariance as a zero-padded version of (a). In (c) we present the result of our algorithm after 5000 iterations with (T--8). In the bottom-right corner of each image (b) and (c) we present a $\times 3$ magnification of some details of the images.}   \label{fig:im_cv}
 \end{figure}
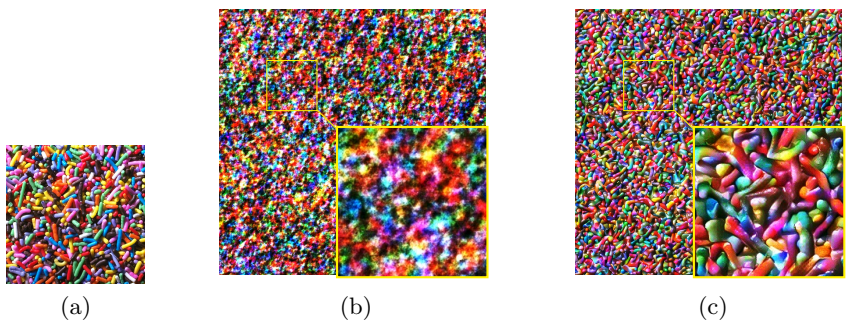
\subsection{Neural network features}
\label{sec:neur-netw-feat-2}
\paragraph{Number of layers.}
We start by investigating the influence of the number of layers in the model by running
the algorithm for different layer configurations. If too many layers are considered the algorithm diverges.
However, if the number of layers considered in the model is reduced we observe different behaviors.
This is illustrated in Figure \ref{fig:dv_layers} where the objective image exhibits strong mid-range structure
information.
\begin{figure}[t]
  \centering
  \hfill
  \subfloat[]{\includegraphics[width=.23\linewidth]{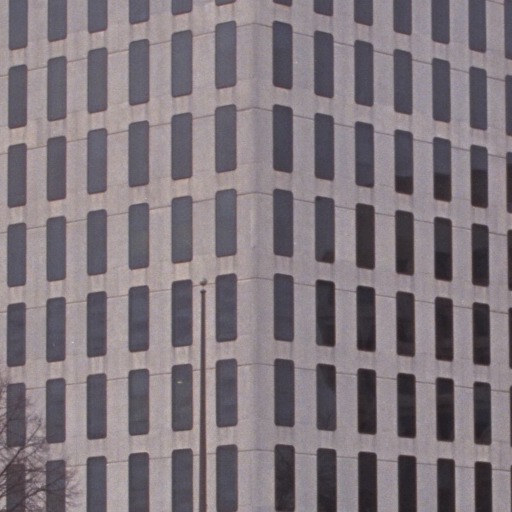}}   
  \hfill
  \subfloat[]{\includegraphics[width=.23\linewidth]{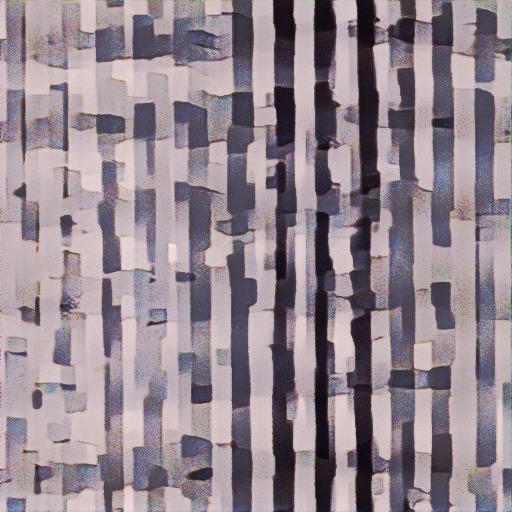}} \hfill
  \subfloat[]{\includegraphics[width=.23\linewidth]{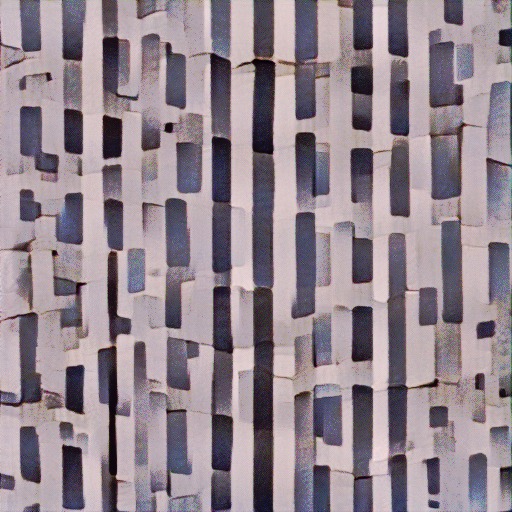}} \hfill
  \subfloat[]{\includegraphics[width=.23\linewidth]{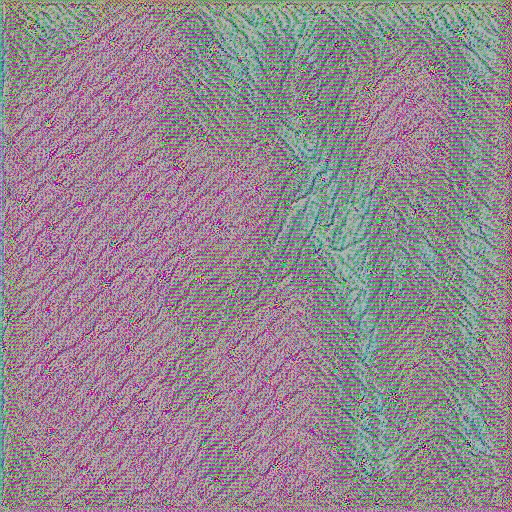}} \hfill
  \caption{\figuretitle{Depth influence} The image in (a) is the original texture. In (b), (c) and (d) we consider the outputs of the sampling algorithms (T--3), (T--6) and (T--8) algorithms. Note that more geometrical structure is retrieved in (c) than in (b) and that the model has diverged in (d).}   \label{fig:dv_layers}
\end{figure}
\paragraph{Model choice.}
In all previous experiments the weights considered in the CNN architecture are pretrained on a classification task as
in \cite{gatys2015texture} and \cite{lu2015learning}. It is natural to ask if such a pretraining is necessary. In
accordance with the results obtained by Gatys et al. \cite{gatys2015texture} we find that a model with no pretraining does not produce
perceptually satisfying texture samples, see Figure \ref{fig:im_noise}. Note that in \cite{ustyuzhaninov2016texture} the authors obtain good results with random convolutional layers and a microcanonical approach.
\begin{figure}[h]
  \centering
  \subfloat[]{\includegraphics[width=.24\linewidth]{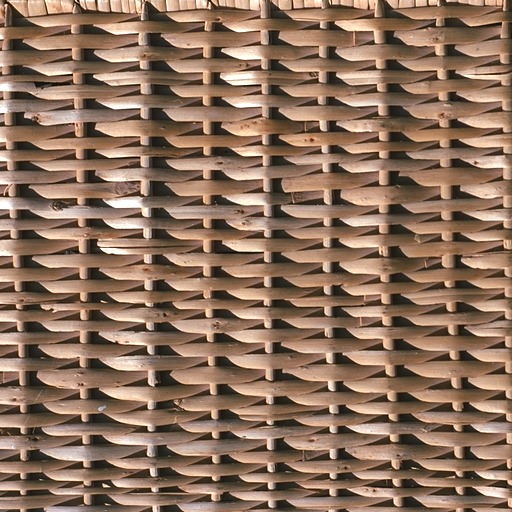}} \hfill
  \subfloat[]{\includegraphics[width=.24\linewidth]{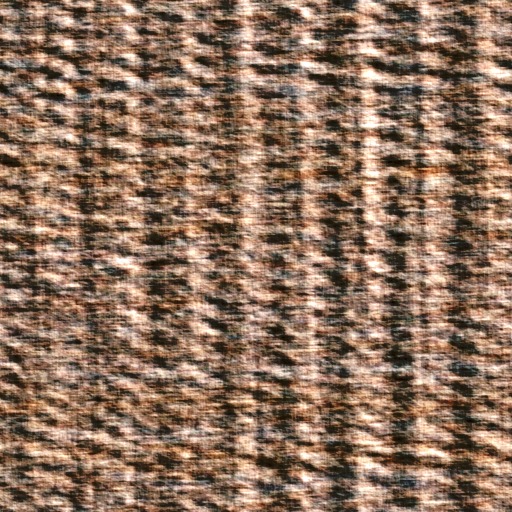}} \hfill
\subfloat[]{\includegraphics[width=.24\linewidth]{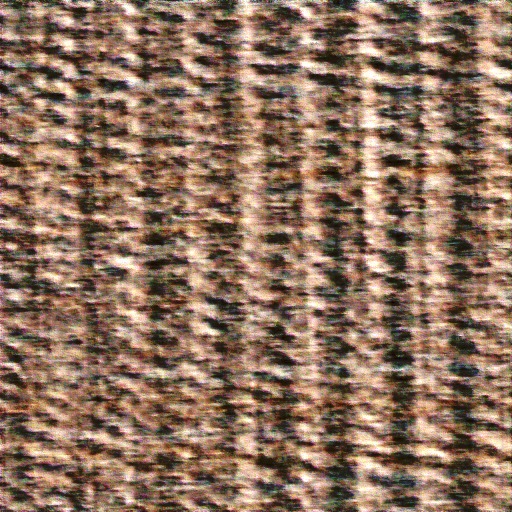}}
  \hfill
  \subfloat[]{\includegraphics[width=.24\linewidth]{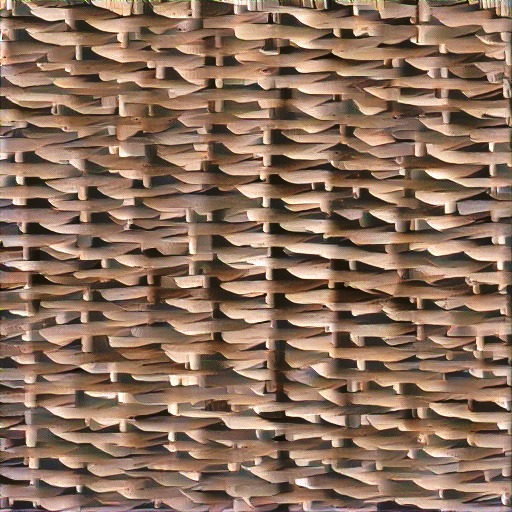}} 
\hfill
  \caption{\figuretitle{Noisy weights} In (a) we present the input image, in (b) the initialization of the algorithm and in (c), respectively (d), the output of the algorithm (G--8), respectively (T--8) after 5000 iterations. Note that no spatial structure is retrieved in (c) which is close to its Gaussian initialization.}   \label{fig:im_noise}
\end{figure}
\subsection{Comparison with state-of-the art methods}
To conclude this experimental study we provide a comparison of our results with state-of-art texture synthesis methods in Figure \ref{fig:im_soa}. Regarding regular textures our model misses certain geometrical constraints, which are encoded by the Gram matrices in \cite{gatys2015texture} for instance. However, our model relies only on 2k features, using (T--8), whereas Gatys et al. use at least 10k parameters.  One way to impose the lost geometrical constraints could be to project the spectrum of the outputs at each step of the algorithm as it was done by Liu et al. \cite{liu2016texture} in a microcanonical model.
\begin{figure}[t]
  \centering
  \subfloat[]{\includegraphics[width=.19\linewidth]{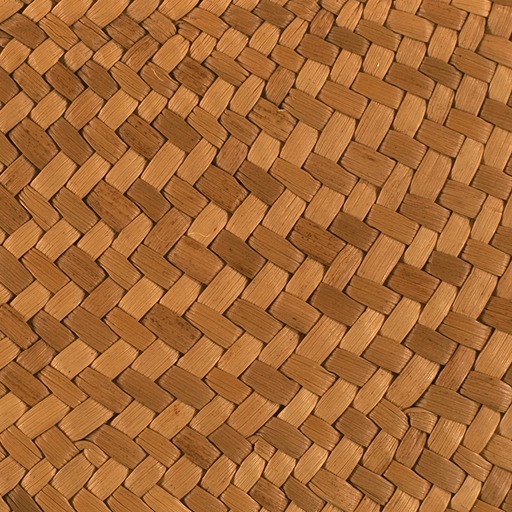}} \hfill
  \subfloat[]{\includegraphics[width=.19\linewidth]{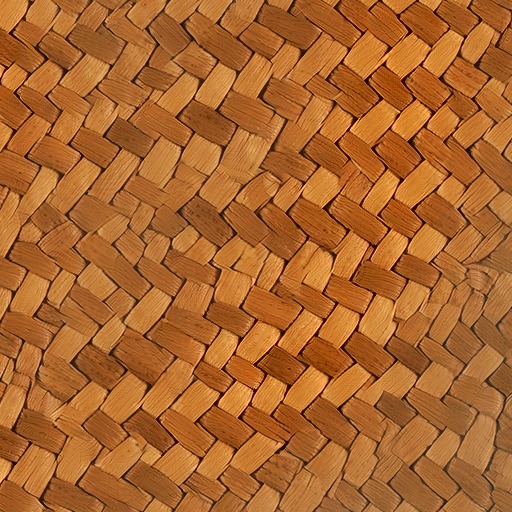}} \hfill
  \subfloat[]{\includegraphics[width=.19\linewidth]{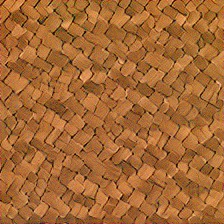}} \hfill
  \subfloat[]{\includegraphics[width=.19\linewidth]{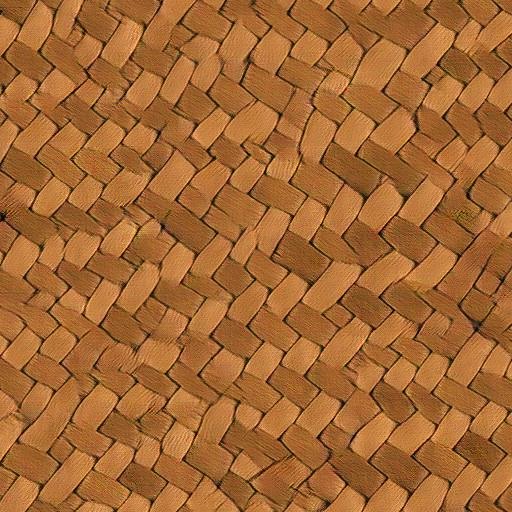}} \hfill
  \subfloat[]{\includegraphics[width=.19\linewidth]{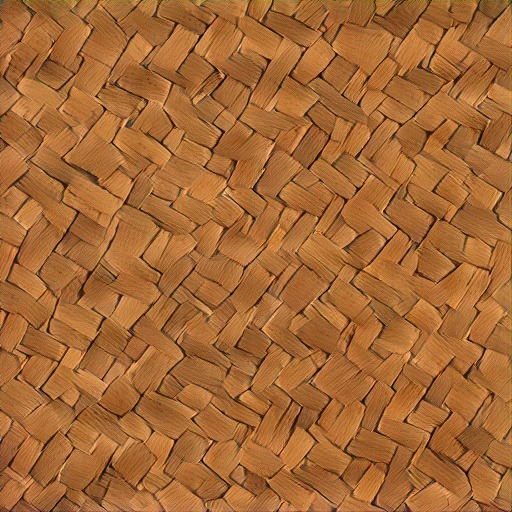}} \hfill
\caption{\figuretitle{Comparison} The input image is shown in (a). In (b) we show the output of the algorithm in \cite{gatys2015texture}, in (c) the output of the DeepFRAME method \cite{lu2015learning} and in (d) we present a result obtained with GAN texture synthesis algorithm \cite{jetchev2016texture}. Our result (e) after 5000 iterations of (T--8) is comparable but lacks spatial organization. Images (a)--(d) extracted from \cite{raad2017survey}.} \label{fig:im_soa}
\end{figure}



\section{Perspectives}
\label{sec:concl-persp}
There still exists a gap between the theoretical analysis of those algorithms which relies
on control theory tools \cite{bruna2018multiscale}, stochastic optimization techniques \cite{atchade2017perturbed} or general state space Markov chain results 
and the experimental study. Indeed the class of functions which is handled by these theoretical results is constrained (regularity assumptions, drift conditions...) and has yet to be extended to more general CNN features. In addition, they scale badly with the dimension of the data which is high in our image processing context. In a future work we wish to extend our theoretical understanding of SOUL algorithms applied to macrocanonical models and draw parallels with the microcanonical results obtained in \cite{bruna2018multiscale}.

%
%
%
\bibliographystyle{splncs04}
\bibliography{research}
\end{document}